\documentclass[reqno]{amsart}

\usepackage{graphicx}
\usepackage{multirow}
\usepackage{hyperref}
\hypersetup{colorlinks=true, linkcolor=blue, urlcolor=blue, citecolor=[rgb]{0, 0.8, 0}, linktocpage=true}

\theoremstyle{plain}
\newtheorem{theorem}{Theorem}[section]

\theoremstyle{remark}
\newtheorem{remark}{Remark}[section]

\numberwithin{equation}{section}
\numberwithin{figure}{section}
\numberwithin{table}{section}

\title[Approximation capability of two hidden layer networks]{Approximation capability of two hidden layer feedforward neural networks with fixed weights}

\author{Namig J. Guliyev}
\address{Institute of Mathematics and Mechanics, Azerbaijan National Academy of Sciences, 9 B.~Vahabzadeh str., AZ1141, Baku, Azerbaijan.}
\email{njguliyev@gmail.com}

\author{Vugar E. Ismailov}
\address{Institute of Mathematics and Mechanics, Azerbaijan National Academy of Sciences, 9 B.~Vahabzadeh str., AZ1141, Baku, Azerbaijan.}
\email{vugaris@mail.ru}



\begin{document}
\maketitle
\begin{abstract}
We algorithmically construct a two hidden layer feedforward neural network (TLFN) model with the weights fixed as the unit coordinate vectors of the $d$-dimensional Euclidean space and having $3d+2$ number of hidden neurons in total, which can approximate any continuous $d$-variable function with an arbitrary precision. This result, in particular, shows an advantage of the TLFN model over the single hidden layer feedforward neural network (SLFN) model, since SLFNs with fixed weights do not have the capability of approximating multivariate functions.
\end{abstract}

\section{Introduction} \label{sec:introduction}

The topic of artificial neural networks is an important and vibrant area of research in modern science. This is due to a large number of application areas. Nowadays, neural networks are being successfully applied in areas as diverse as computer science, finance, medicine, geology, engineering, physics, etc. Perhaps the greatest advantage of neural networks is their ability to be used as an arbitrary function approximation mechanism. In this paper, we are interested in questions of density (or approximation with arbitrary accuracy) of the multilayer feedforward neural network (MLFN) model. Approximation capabilities of this model have been well studied for the past 30 years. Choosing various activation functions $\sigma$ it was shown in a great number of papers that MLFNs can approximate any continuous function with an arbitrary precision. The most simple MLFN model is the single hidden layer feedforward neural network (SLFN) model. This model evaluates a multivariate function
\begin{equation} \label{eq:single}
  \sum_{i=1}^{k} c_{i} \sigma(\mathbf{w}^{i} \cdot \mathbf{x} - \theta_{i})
\end{equation}
of the variable $\mathbf{x} = (x_1, \ldots, x_d)$, $d \ge 1$. Here the weights $\mathbf{w}^{i}$ are vectors in $\mathbb{R}^{d}$, the thresholds $\theta_{i}$ and the coefficients $c_{i}$ are real numbers, and the activation function $\sigma$ is a univariate function. A multiple hidden layer network is defined by iterations of the SLFN model. For example, the output of the two hidden layer feedforward neural network (TLFN) model with $k$ units in the first layer, $m$ units in the second layer and the input $\mathbf{x} = (x_1, \ldots, x_d)$ is
\begin{equation*}
  \sum_{i=1}^{m} e_{i} \sigma \left( \sum_{j=1}^{k} c_{ij} \sigma(\mathbf{w}^{ij} \cdot \mathbf{x} - \theta_{ij}) - \zeta_{i} \right).
\end{equation*}
Here $d_{i}$, $c_{ij}$, $\theta_{ij}$ and $\gamma_{i}$ are real numbers, $\mathbf{w}^{ij}$ are vectors of $\mathbb{R}^{d}$, and $\sigma $ is a fixed univariate function.

In many applications, it is convenient to take an activation function $\sigma$ as a \emph{sigmoidal function}, which is defined as
\begin{equation*}
  \lim_{t \to -\infty} \sigma(t) = 0 \qquad \text{ and } \qquad \lim_{t \to +\infty} \sigma(t) = 1.
\end{equation*}
The literature on neural networks abounds with the use of such functions and their superpositions.

The possibility of approximating a continuous function on a compact subset of $\mathbb{R}^{d}$, $d \ge 1$, by SLFNs with a sigmoidal activation function has been tremendously studied in many papers. To the best of our knowledge, Gallant and White~\cite{GW88} were the first to prove the universal approximation property for the SLFN model with a sigmoidal activation function. Their activation function, called the \emph{cosine squasher}, has the ability to generate any trigonometric series. As such, this function has the density property. Carroll and Dickinson~\cite{CD89} implemented the inverse Radon transformation to approximate $L^{2}$ functions, using any continuous sigmoidal function as an activation function. Cybenko~\cite{C89} proved that SLFNs with a continuous sigmoidal activation function can approximate any continuous function with arbitrary accuracy on compact subsets of $\mathbb{R}^{d}$. Funahashi~\cite{F89}, independently of Cybenko, proved the density property for a continuous monotone sigmoidal function. Hornik, Stinchcombe and White~\cite{HSW89} proved density of SLFNs with a discontinuous bounded sigmoidal function. K\r{u}rkov\'{a}~\cite{K92} showed that staircase-like functions of any sigmoidal type has the capability of approximating continuous univariate functions on any compact subset of $\mathbb{R}$ within arbitrarily small tolerance. This result was substantially used in K\r{u}rkov\'{a}'s further results, which showed that a continuous multivariate function can be approximated arbitrarily well by TLFNs with a sigmoidal activation function (see~\cite{K91, K92}). Chen, Chen and Liu~\cite{CCL92} generalized the result of Cybenko by proving that any continuous function on a compact subset of $\mathbb{R}^{d}$ can be approximated by SLFNs with a bounded (not necessarily continuous) sigmoidal activation function. Almost the same result was independently obtained by Jones~\cite{J90}. Costarelli and Spigler~\cite{CS13} constructed special sums of the form (\ref{eq:single}), using a given function $f\in C[a,b]$. They then proved that these sums approximate $f$ within any degree of accuracy. In their result, similar to \cite{CCL92}, $\sigma$ is any bounded sigmoidal function. Chui and Li~\cite{CL92} proved that SLFNs with a continuous sigmoidal activation function having integer weights and thresholds can approximate continuous univariate functions on any compact subset of the real line.

In a number of subsequent papers, which considered the density problem for the SLFN model, nonsigmoidal activation functions were allowed. Here we cite a few of them. The papers by Stinchcombe and White~\cite{SW90}, Cotter~\cite{C90}, Hornik~\cite{H91}, Mhaskar and Micchelli~\cite{MM92} are among many others. It should be remarked that the more general result in this direction belongs to Leshno, Lin, Pinkus and Schocken~\cite{LLPS93}. They proved that the necessary and sufficient condition for any continuous activation function to have the density property is that it not be a polynomial. For more detailed discussion of the density problem, see the review paper by Pinkus~\cite{P99}.

The above results show that SLFNs with various activation functions enjoy the universal approximation property. In recent years, the theory of neural networks has been developed further in this direction. For example, from the point of view of practical applications, SLFNs with a restricted set of weights have gained a special interest (see, e.g., \cite{D02, I12, I15, IS17, JYJ10, LFN04}). It was proved that SLFNs with some restricted set of weights still possess the universal approximation property. For example, Stinchcombe and White~\cite{SW90} showed that SLFNs with a polygonal, polynomial spline or analytic activation function and a bounded set of weights have the universal approximation property. Ito~\cite{I91, I92} investigated this property of networks using monotone sigmoidal functions, with only weights located on the unit sphere. In \cite{I12, I15, IS17}, the second coauthor considered SLFNs with weights varying on a restricted set of directions, and gave several necessary and sufficient conditions for good approximation by such networks. For a set of weights consisting of two directions, he showed that there is a geometrically explicit solution to the problem. Hahm and Hong \cite{HH04}~went further in this direction, and showed that SLFNs with fixed weights can approximate arbitrarily well any continuous univariate function. Since fixed weights reduce the computational expense and training time, this result is of particular interest. In a mathematical formulation, the result says that for a bounded measurable sigmoidal function $\sigma$, networks of the form $\sum_{i=1}^{k} c_{i} \sigma(\alpha x - \theta_{i})$ are dense in $C[a, b]$. Cao and Xie~\cite{CX10} strengthened this result by specifying the number of hidden neurons to realize $\varepsilon$-approximation to any continuous function. By implementing modulus of continuity, they established Jackson-type upper bound estimations for the approximation error.

Approximation capabilities of SLFNs with fixed weights were also analyzed in Lin, Guo, Cao and Xu~\cite{LGCX13}. Taking the activation function $\sigma$ as a continuous, even and $2\pi$-periodic function, the authors of \cite{LGCX13} showed that neural networks of the form $\sum_{i=1}^{r} c_{i} \sigma(x-x_{i})$ can approximate any continuous function on $[-\pi, \pi]$ with an arbitrary precision $\varepsilon$. Note that all the weights are fixed equal to $1$, and consequently do not depend on $\varepsilon$. To prove this, they first gave an integral representation for trigonometric polynomials, and constructed explicitly a network with the weight $1$ that approximates this integral representation. Finally, the obtained result for trigonometric polynomials was used to prove a Jackson-type upper bound for the approximation error.

Note that SLFNs with a fixed number of weights cannot approximate $d$-variable functions if $d > 1$. That is, if in (\ref{eq:single}) we have $n$ different weights $\mathbf{w}^{i}$ ($n$ is fixed), then there exist a compact set $Q \subset \mathbb{R}^{d}$ and a function $f\in C(Q)$, which cannot be approximated arbitrarily well by the networks formed as (\ref{eq:single}). This follows from a result of Lin and Pinkus on sums of $n$ ridge functions (see \cite[Theorem 5.1]{LP93}). For details, see our recent paper \cite{GI18}. Thus the above results of Hahm and Hong~\cite{HH04}, Cao and Xie~\cite{CX10}, Lin, Guo, Cao and Xu~\cite{LGCX13} cannot be generalized to the $d$-dimensional case if one allows only the SLFN model of neural networks.

It should be remarked that in all of the above-mentioned works the number of neurons $k$ in the hidden layer is not fixed. As such to achieve a desired precision one may take an excessive number of hidden neurons. Unfortunately, practicality decreases with the increase of the number of neurons in the hidden layer. In other words, SLFNs are not always effective if the number of neurons in the hidden layer is prescribed. More precisely, they are effective if and only if we consider univariate functions. In~\cite{GI16}, we consider constructive approximation on any finite interval of $\mathbb{R}$ by SLFNs with a fixed number of hidden neurons. We construct algorithmically a smooth, sigmoidal, almost monotone activation function $\sigma$ providing approximation to an arbitrary univariate continuous function within any degree of accuracy. Note that the result of~\cite{GI16} is not applicable to multivariate functions.

The first crucial step in investigating approximation capabilities of MLFNs with a prescribed number of hidden neurons was made by Maiorov and Pinkus~\cite{MP99}. Their remarkable result revealed that TLFNs with $3d$ units in the first layer and $6d + 3$ units in the second layer can approximate an arbitrary continuous $d$-variable function. Using a different activation function than in~\cite{MP99}, the second coauthor~\cite{I14} showed that the number of neurons in hidden layers can be reduced to $d$ and $2d + 2$ respectively. Note that the results of both papers carry a theoretical character, as they indicate only the existence of the corresponding TLFNs, their activation functions.

We see that in each result above at least one of the following general properties is violated.
\begin{enumerate}
  \item the number of hidden neurons is fixed;
  \item the weights are fixed;
  \item the activation function is computable;
  \item the network has the capability of approximating $d$-variable functions in the case $d > 1$.
\end{enumerate}
In this paper, we construct a special TLFN model that satisfies all of the properties (1)--(4). In addition, we show that along with the number of hidden neurons and weights, it is also possible to fix some dilation coefficients of the constructed activation function.

\section{The main result} \label{sec:result}

In the sequel, we deal with an activation function, which is monotonic in the weak sense. Here by \emph{weak monotonicity} we understand behavior of a function whose difference in absolute value from a monotonic function is a sufficiently small number. In this regard we say that a real function $f$ defined on a set $X \subseteq \mathbb{R}$ is \emph{$\lambda$-increasing} (respectively, \emph{$\lambda $-decreasing}) if there exists an increasing (respectively, decreasing) function $u \colon X \to \mathbb{R}$ such that $|f(x) - u(x)| \le \lambda$ for all $x \in X$. Clearly, $0$-monotonicity coincides with the usual concept of monotonicity and a $\lambda_1$-increasing function is $\lambda_{2}$-increasing if $\lambda_1 \le \lambda_{2}$.

Our main result is the following theorem.
\begin{theorem} \label{thm:main}
Assume a closed interval $[a,b] \subset \mathbb{R}$ is given, $s = b - a$, and $\lambda $ is any sufficiently small positive real number. Then one can algorithmically construct a computable, infinitely differentiable, sigmoidal activation function $\sigma \colon \mathbb{R} \to \mathbb{R}$ which is strictly increasing on $(-\infty, s)$, $\lambda$-strictly increasing on $[s, +\infty)$ and satisfies the following property: For any continuous function $f$ on the $d$-dimensional box $[a,b]^{d}$ and $\varepsilon >0,$ there exist constants $e_p$, $c_{pq}$, $\theta_{pq}$ and $\zeta_p$ such that the inequality
\begin{equation*}
  \left| f(\mathbf{x}) - \sum_{p=1}^{2d+2} e_p \sigma \left( \sum_{q=1}^{d} c_{pq} \sigma(\mathbf{w}^{q} \cdot \mathbf{x} - \theta_{pq}) - \zeta_p \right) \right| < \varepsilon
\end{equation*}
holds for all $\mathbf{x} = (x_1, \ldots, x_d) \in [a, b]^{d}$. Here the weights $\mathbf{w}^{q}$, $q = 1, \ldots, d$, are fixed as follows:
\begin{equation*}
  \mathbf{w}^{1} = (1, 0, \ldots, 0), \quad \mathbf{w}^{2} = (0, 1, \ldots, 0), \quad \ldots, \quad \mathbf{w}^{d} = (0, 0, \ldots, 1).
\end{equation*}
In addition, all the coefficients $e_p$, except one, are equal.
\end{theorem}
\begin{proof}
We start with the algorithmic construction of $\sigma$ mentioned in the theorem. The algorithm consists of the following steps.

1. Consider the function
\begin{equation*}
  h(x) := 1 - \frac{\min\{1/2, \lambda\}}{1 + \log(x - s + 1)}.
\end{equation*}
Obviously, this function is strictly increasing on the real line and satisfies the following properties:
\begin{enumerate}
  \item $0 < h(x) < 1$ for all $x \in [s, +\infty)$;
  \item $1 - h(s) \le \lambda$;
  \item $h(x) \to 1$ as $x \to +\infty$.
\end{enumerate}
Our purpose is to construct $\sigma$ satisfying the two-sided inequality
\begin{equation} \label{eq:h_sigma_1}
  h(x) < \sigma(x) < 1
\end{equation}
for $x \in [s, +\infty)$. Then our $\sigma$ will approach $1$ as $x$ approaches $+\infty$ and obey the inequality
\begin{equation*}
  |\sigma(x) - h(x)| \le \lambda,
\end{equation*}
that is, it will be a $\lambda$-increasing function.

2. In this step, we enumerate the monic polynomials with rational coefficients. Let $q_n$ be the Calkin--Wilf sequence (see~\cite{CW00}). We can enumerate all the rational numbers by setting
$$r_0 := 0, \quad r_{2n} := q_n, \quad r_{2n-1} := -q_n, \ n = 1, 2, \dots.$$
Note that each monic polynomial with rational coefficients can uniquely be written as $r_{k_0} + r_{k_1} x + \ldots + r_{k_{l-1}} x^{l-1} + x^l$, and each positive rational number determines a unique finite continued fraction
$$
  [m_0; m_1, \ldots, m_l] := m_0 + \dfrac1{m_1 + \dfrac1{m_2 + \dfrac1{\ddots + \dfrac1{m_l}}}}
$$
with $m_0 \ge 0$, $m_1, \ldots, m_{l-1} \ge 1$ and $m_l \ge 2$. We now construct a one-to-one mapping between the set of all monic polynomials with rational coefficients and the set of all positive rational numbers as follows. To the only zeroth-degree monic polynomial 1 we associate the rational number 1, to each first-degree monic polynomial of the form $r_{k_0} + x$ we associate the rational number $k_0 + 2$, to each second-degree monic polynomial of the form $r_{k_0} + r_{k_1} x + x^2$ we associate the rational number $[k_0; k_1 + 2] = k_0 + 1 / (k_1 + 2)$, and to each monic polynomial
$$
  r_{k_0} + r_{k_1} x + \ldots + r_{k_{l-2}} x^{l-2} + r_{k_{l-1}} x^{l-1} + x^l
$$
of degree $l \ge 3$ we associate the rational number $[k_0; k_1 + 1, \ldots, k_{l-2} + 1, k_{l-1} + 2]$. In other words, we define $u_1(x) := 1$,
$$
  u_n(x) := r_{q_n-2} + x
$$
if $q_n \in \mathbb{Z}$,
$$
  u_n(x) := r_{m_0} + r_{m_1-2} x + x^2
$$
if $q_n = [m_0; m_1]$, and
$$
  u_n(x) := r_{m_0} + r_{m_1-1} x + \ldots + r_{m_{l-2}-1} x^{l-2} + r_{m_{l-1}-2} x^{l-1} + x^l
$$
if $q_n = [m_0; m_1, \ldots, m_{l-2}, m_{l-1}]$ with $l \ge 3$. Hence the first few elements of this sequence are defined as
$$
  1, \quad x^2, \quad x, \quad x^2 - x, \quad x^2 - 1, \quad x^3, \quad x - 1, \quad x^2 + x, \quad \ldots.
$$
The sequence of monic polynomials will be used in the sequel. 

3. First we construct $\sigma$ on the intervals $[(2n-1)s, 2ns]$, $n = 1, 2, \ldots$. For each monic polynomial $u_n(x) = \rho_0 + \rho_1 x + \ldots + \rho_{l-1} x^{l-1} + x^l$ with rational coefficients, set
\begin{equation*}
  B_1 := \rho_0 + \frac{\rho_1-|\rho_1|}{2} + \ldots + \frac{\rho_{l-1} - |\rho_{l-1}|}{2}
\end{equation*}
and
\begin{equation*}
  B_2 := \rho_0 + \frac{\rho_1+|\rho_1|}{2} + \ldots + \frac{\rho_{l-1} + |\rho_{l-1}|}{2} + 1.
\end{equation*}
Note that the numbers $B_1$ and $B_2$ depend on $n$, but for simplicity we will omit this in the notation.

Consider the sequence
\begin{equation*}
  M_n := h((2n+1)s), \qquad n = 1, 2, \ldots.
\end{equation*}
Obviously, this sequence is strictly increasing and converges to $1$. 

Now we define $\sigma$ as the function
\begin{equation} \label{eq:sigma_again}
\sigma(x) := a_n + b_n u_n \left( \frac{x}{s} - 2n + 1 \right), \quad x \in [(2n-1)s, 2ns].
\end{equation}
Here
\begin{equation} \label{eq:a_1, b_1}
  a_1 := \frac{1}{2}, \qquad b_1 := \frac{h(3s)}{2},
\end{equation}
and
\begin{equation} \label{eq:a_n, b_n}
  a_n := \frac{(1 + 2M_n) B_2 - (2 + M_n) B_1}{3(B_2 - B_1)}, \qquad b_n := \frac{1 - M_n}{3(B_2 - B_1)}, \qquad n = 2, 3, \ldots.
\end{equation}

It is not difficult to see that for $n>2$ the numbers $a_n$, $b_n$ are the coefficients of the linear function $y = a_n + b_n x$ mapping the closed interval $[B_1, B_2]$ onto the closed interval $[(1+2M_n)/3, (2+M_n)/3]$. In addition, for $n=1$, i.e. on the interval $[s,2s]$,
\begin{equation*}
  \sigma(x) = \frac{1+M_1}{2}.
\end{equation*}
Thus, we obtain that
\begin{equation} \label{eq:h_M_sigma_1}
  h(x) < M_n < \frac{1+2M_n}{3} \le \sigma(x) \le \frac{2+M_n}{3} < 1,
\end{equation}
for all $x \in [(2n-1)s, 2ns]$, $n = 1$, $2$, $\ldots$.

4. In this step, we construct $\sigma$ on the intervals $[2ns, (2n+1)s]$, $n = 1, 2, \ldots$. To this end we use the \emph{smooth transition function}
\begin{equation*}
  \beta_{a,b}(x) := \frac{\widehat{\beta}(b-x)}{\widehat{\beta}(b-x) + \widehat{\beta}(x-a)},
\end{equation*}
where
\begin{equation*}
  \widehat{\beta}(x) := \begin{cases} e^{-1/x}, & x > 0, \\ 0, & x \le 0. \end{cases}
\end{equation*}
Clearly, $\beta_{a,b}(x) = 1$ for $x \le a$, $\beta_{a,b}(x) = 0$ for $x \ge b$, and $0 < \beta_{a,b}(x) < 1$ for $a < x < b$. 

Consider the sequence
\begin{equation*}
  K_n := \frac{\sigma(2ns) + \sigma((2n+1)s)}{2}, \qquad n = 1, 2, \ldots.
\end{equation*}
Recall that the numbers $\sigma(2ns)$ and $\sigma((2n+1)s)$ have already been defined in the previous step. Since both the numbers $\sigma(2ns)$ and $\sigma((2n+1)s)$ belong to the interval $(M_n, 1)$, it follows that $K_n \in (M_n, 1)$.

First we extend $\sigma$ smoothly to the interval $[2ns, 2ns + s/2]$. Take the number $\varepsilon := (1 - M_n)/6$ and select $\delta \le s/2$ such that
\begin{equation} \label{eq:epsilon}
  \left| a_n + b_n u_n \left( \frac{x}{s} - 2n + 1 \right) - \left( a_n + b_n u_n(1) \right) \right| \le \varepsilon, \quad x \in [2ns, 2ns + \delta].
\end{equation}
One can select this $\delta$ as
\begin{equation*}
  \delta := \min\left\{ \frac{\varepsilon s}{b_n C}, \frac{s}{2} \right\},
\end{equation*}
where $C > 0$ is any number satisfying $|u'_n(x)| \le C$ for $x \in
(1, 1.5)$. For example, if $n=1$, then $\delta$ can be selected as $s/2$. Now define $\sigma$ on the left-hand half of the interval $[2ns, (2n+1)s]$ as the function
\begin{equation} \label{eq:sigma_left}
\begin{split}
  \sigma(x) & := K_n - \beta_{2ns, 2ns + \delta}(x) \\
  & \times \left(K_n - a_n - b_n u_n \left( \frac{x}{s} - 2n + 1 \right)\right), \quad x \in \left[ 2ns, 2ns + \frac{s}{2} \right].
\end{split}
\end{equation}

Let us prove that $\sigma(x)$ satisfies the condition~(\ref{eq:h_sigma_1}). Indeed, if $2ns + \delta \le x \le 2ns + s/2$, then there is nothing to prove, since $\sigma(x) = K_n \in (M_n, 1)$. If $2ns \le x < 2ns + \delta$, then $0 < \beta_{2ns, 2ns+\delta}(x) \le 1$ and hence from~(\ref{eq:sigma_left}) we obtain that for each $x \in [2ns, 2ns + \delta)$, $\sigma(x)$ is between the numbers $K_n$ and $A_n(x) := a_n + b_n u_n \left( \frac{x}{s} - 2n + 1 \right)$. On the other hand, from~(\ref{eq:epsilon}) it follows that
\begin{equation*}
  a_n + b_n u_n(1) - \varepsilon \le A_n(x) \le a_n + b_n u_n(1) + \varepsilon.
\end{equation*}
The last inequality together with~(\ref{eq:sigma_again}) and the inequalities~(\ref{eq:h_M_sigma_1}) yields that $A_n(x) \in \left[ \frac{1+2M_n}{3} - \varepsilon, \frac{2+M_n}{3} + \varepsilon \right]$ for $x \in [2ns, 2ns + \delta)$. Since $\varepsilon = (1 - M_n)/6$, the inclusion $A_n(x) \in (M_n, 1)$ is valid. Now since both $K_n$ and $A_n(x)$ lie in the interval $(M_n, 1)$, we conclude that
\begin{equation*}
  h(x) < M_n < \sigma(x) < 1, \quad \text{for } x \in \left[ 2ns, 2ns + \frac{s}{2} \right].
\end{equation*}

We define $\sigma$ on the right-hand half of the interval in a similar way:
\begin{equation*}
\begin{split}
  \sigma(x) & := K_n - (1 - \beta_{(2n+1)s - \overline{\delta}, (2n+1)s}(x)) \\
  & \times \left(K_n - a_{n+1} - b_{n+1} u_{n+1} \left( \frac{x}{s} - 2n - 1 \right)\right), \quad x \in \left[ 2ns + \frac{s}{2}, (2n+1)s \right],
\end{split}
\end{equation*}
where
\begin{equation*}
  \overline{\delta} := \min\left\{ \frac{\overline{\varepsilon}s}{b_{n+1} \overline{C}}, \frac{s}{2} \right\}, \qquad \overline{\varepsilon} := \frac{1 - M_{n+1}}{6}, \qquad \overline{C} \ge \sup_{[-0.5, 0]} |u'_{n+1}(x)|.
\end{equation*}
It is not difficult to verify, as above, that the constructed $\sigma(x)$ satisfies the condition~(\ref{eq:h_sigma_1}) on $[2ns + s/2, 2ns + s]$ and
\begin{equation*}
  \sigma \left( 2ns + \frac{s}{2} \right) = K_n, \qquad \sigma^{(i)} \left( 2ns + \frac{s}{2} \right) = 0, \quad i = 1, 2, \ldots.
\end{equation*}

Steps 3 and 4 together construct $\sigma$ on the interval $[s, +\infty)$.

5. On the remaining interval $(-\infty, s)$, we define $\sigma$ as
\begin{equation*}
  \sigma(x) := \left( 1 - \widehat{\beta}(s-x) \right) \frac{1 + M_1}{2}, \quad x \in (-\infty, s).
\end{equation*}
Clearly, $\sigma$ is a strictly increasing, smooth function on $(-\infty, s)$. In addition, $\sigma(x) \to \sigma(s) = (1 + M_1) / 2$, as $x$ tends to $s$ from the left and $\sigma^{(i)}(s) = 0$ for $i = 1$, $2$, $\ldots$. This final step completes the construction of $\sigma$ on the whole real line. Note that the constructed $\sigma$ is sigmoidal, infinitely differentiable on $\mathbb{R}$, strictly increasing on $(-\infty, s)$ and $\lambda$-strictly increasing on $[s, +\infty)$.

It should be noted that the above algorithm allows one to compute $\sigma$ at any point of the real axis instantly. The code of this algorithm is available at \url{https://sites.google.com/site/njguliyev/papers/tlfn}. As a practical example, we give here the graph of $\sigma$ (see Figure~\ref{fig:sigma50}) and a numerical table (see Table~\ref{tbl:sigma}) containing several computed values of this function on the interval $[0, 50]$. All computations were done in SageMath~\cite{Sage}. Figure~\ref{fig:sigma100_lambda} shows how the graph of the $\lambda$-increasing function $\sigma$ changes on the interval $[0, 100]$ as the parameter $\lambda$ decreases. Figure~\ref{fig:sigma100_s} displays variations in the graph of $\sigma$ with respect to the length $s$ of a closed interval $[a, b]$.

\begin{figure}
  \includegraphics[width=1.0\textwidth]{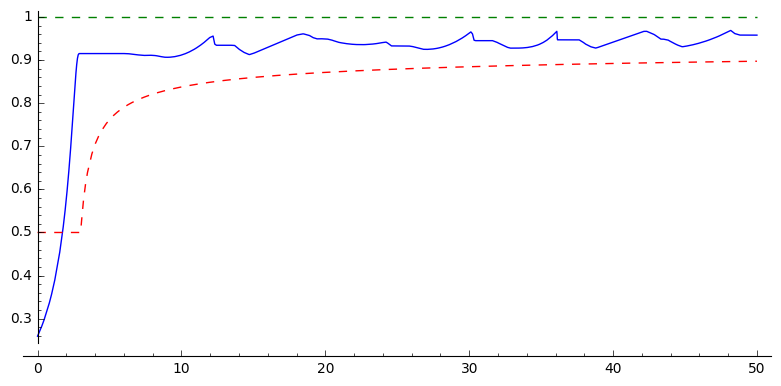}
  \caption{The graph of $\sigma$ on $[0, 50]$ ($s = 3$, $\lambda = 1/2$)}
  \label{fig:sigma50}
\end{figure}

\begin{table}
  \caption{Some computed values of $\sigma$ ($s = 3$, $\lambda = 1/2$)}
  \label{tbl:sigma}
  \begin{tabular}{|c|c|c|c|c|c|c|c|c|c|} \hline
    $t$ & $\sigma$ & $t$ & $\sigma$ & $t$ & $\sigma$ & $t$ & $\sigma$ & $t$ & $\sigma$ \\ \hline
    $0$ & $0.25941$ & $10$ & $0.91169$ & $20$ & $0.94932$ & $30$ & $0.96241$ & $40$ & $0.94166$ \\ \hline
    $1$ & $0.36008$ & $11$ & $0.92728$ & $21$ & $0.94074$ & $31$ & $0.94506$ & $41$ & $0.95333$ \\ \hline
    $2$ & $0.57848$ & $12$ & $0.95325$ & $22$ & $0.93635$ & $32$ & $0.94003$ & $42$ & $0.96499$ \\ \hline
    $3$ & $0.91514$ & $13$ & $0.93437$ & $23$ & $0.93635$ & $33$ & $0.92771$ & $43$ & $0.95602$ \\ \hline
    $4$ & $0.91514$ & $14$ & $0.92551$ & $24$ & $0.94074$ & $34$ & $0.92905$ & $44$ & $0.94295$ \\ \hline
    $5$ & $0.91514$ & $15$ & $0.91549$ & $25$ & $0.93278$ & $35$ & $0.93842$ & $45$ & $0.93186$ \\ \hline
    $6$ & $0.91514$ & $16$ & $0.92958$ & $26$ & $0.93177$ & $36$ & $0.96385$ & $46$ & $0.93943$ \\ \hline
    $7$ & $0.91198$ & $17$ & $0.94366$ & $27$ & $0.92482$ & $37$ & $0.94692$ & $47$ & $0.95079$ \\ \hline
    $8$ & $0.91105$ & $18$ & $0.95775$ & $28$ & $0.92900$ & $38$ & $0.93923$ & $48$ & $0.96593$ \\ \hline
    $9$ & $0.90650$ & $19$ & $0.95532$ & $29$ & $0.94153$ & $39$ & $0.92999$ & $49$ & $0.95800$ \\ \hline
  \end{tabular}
\end{table}

\begin{figure}
  \includegraphics[width=1.0\textwidth]{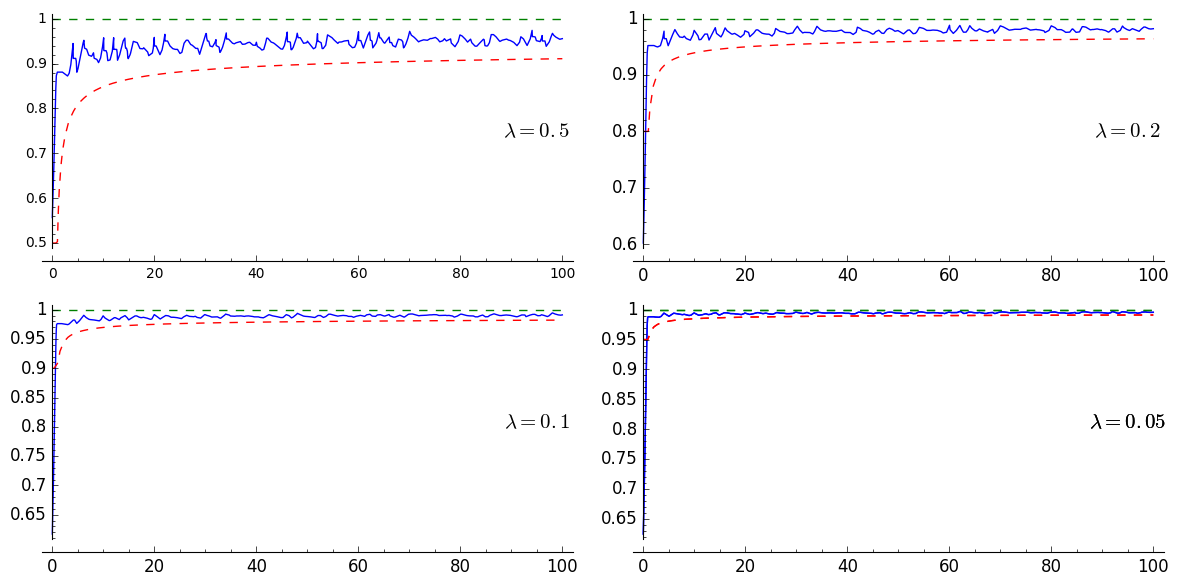}
  \caption{Changes in the graph of $\sigma$ with respect to $\lambda$ ($s = 1$)}
  \label{fig:sigma100_lambda}
\end{figure}

\begin{figure}
  \includegraphics[width=1.0\textwidth]{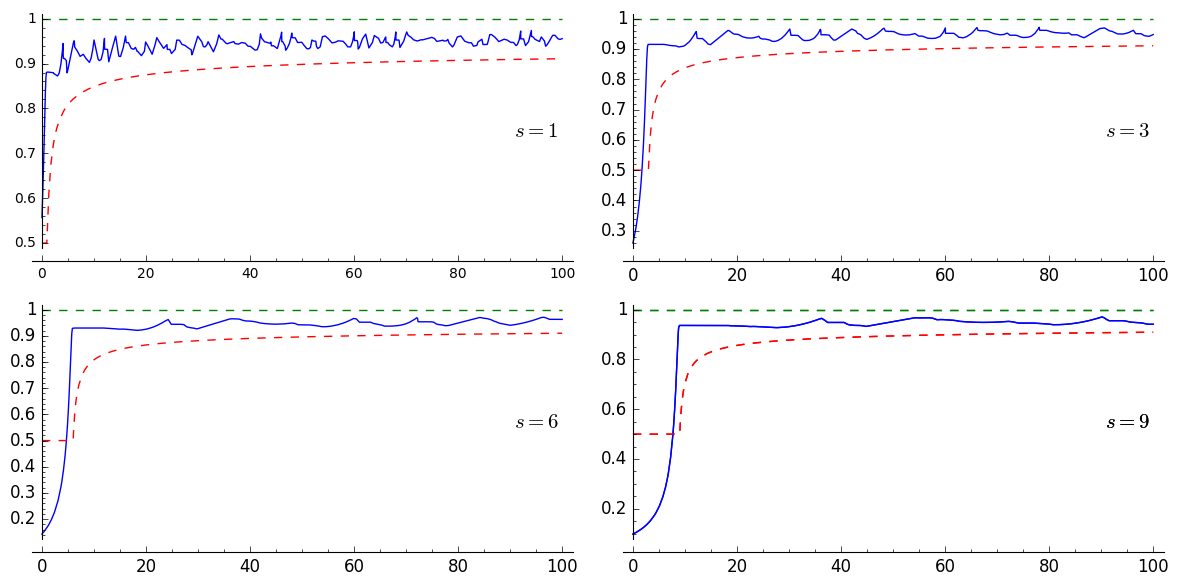}
  \caption{Changes in the graph of $\sigma$ with respect to $s$ ($\lambda = 0.75$)}
  \label{fig:sigma100_s}
\end{figure}

Now we show that in addition to its nice properties such as computability, smoothness and weak monotonicity, our $\sigma$ enjoys an important property of approximating each continuous $d$-variable function as an activation function for TLFNs with a fixed number of hidden neurons.

It follows from~(\ref{eq:sigma_again}) that
\begin{equation} \label{eq:sigma_d}
  \sigma(s x + (2n-1) s) = a_n + b_n u_n(x), \qquad x \in [0, 1]
\end{equation}
for $n = 1$, $2$, $\ldots$. Here $a_n$ and $b_n$ are computed by~(\ref{eq:a_1, b_1}) and~(\ref{eq:a_n, b_n}) for $n = 1$ and $n > 1$ respectively. From~(\ref{eq:sigma_d}) we obtain that each monic polynomial $u_n$, $n = 1$, $2$, $\ldots$, can be represented in the form
\begin{equation} \label{eq:u}
  u_n(x) = \frac{1}{b_n} \sigma(s x + (2n - 1) s) - \frac{a_n}{b_n}.
\end{equation}

Let now $f$ be any continuous function on the box $[a, b]^d$. By the Kolmogorov superposition theorem \cite{K57} in the form given by Lorentz \cite{L65} and Sprecher \cite{S65}, there exist constants $\lambda_q > 0$, $q = 1$, $\ldots$, $d$ with $\sum_{q=1}^{d} \lambda_q = 1$ and nondecreasing continuous functions $\phi_p \colon [a, b] \to [0, 1]$, $p = 1$, $\ldots$, $2d + 1$ such that every continuous function $f \colon [a, b]^d \to \mathbb{R}$ admits the representation
\begin{equation} \label{eq:Kolmogorov}
  f(x_1, \ldots, x_d) = \sum_{p=1}^{2d+1} g \left( \sum_{q=1}^d \lambda_q \phi_p(x_q) \right)
\end{equation}
for some $g \in C[0, 1]$ depending on $f$.

By the density of polynomials with the rational coefficients in the space of continuous functions over any compact subset of $\mathbb{R}$, for the exterior continuous univariate function $g$ in (\ref{eq:Kolmogorov}) and any $\varepsilon > 0$ there exists a polynomial $p(x)$ of the mentioned form such that
\begin{equation*}
  |g(x) - p(x)| < \frac{\varepsilon}{2(2d+1)}
\end{equation*}
for all $x \in [0, 1]$. Denote by $p_0$ the leading coefficient of $p$. If $p_0\neq 0$ (i.e., $p\not\equiv 0$) then we define $u_{n}$ as $u_{n}(x):=p(x)/p_0$, otherwise we just set $u_{n}(x):=1$. In both cases
\begin{equation*}
  |g(x) - p_0 u_n(x)| < \frac{\varepsilon}{2(2d+1)}, \qquad x \in [0, 1].
\end{equation*}
This together with~(\ref{eq:u}) means that
\begin{equation} \label{eq:g}
  |g(x) - (\alpha_0 \sigma(s x - \beta_0) - \gamma_0)| < \frac{\varepsilon}{2(2d+1)}
\end{equation}
for some $\alpha_0$, $\beta_0$, $\gamma_0 \in \mathbb{R}$ and all $x \in [0,1]$. Namely,
\begin{equation} \label{eq:coef}
  \alpha_0 = \frac{p_0}{b_n}, \qquad \beta_0 = s - 2ns, \qquad \gamma_0 = \frac{p_0 a_n}{b_n}.
\end{equation}
Substituting (\ref{eq:g}) in (\ref{eq:Kolmogorov}) we obtain that
\begin{equation} \label{eq:f}
  \left\vert f(x_1, \ldots, x_d) - \sum_{p=1}^{2d+1} \left( \alpha_0 \sigma \left( s \sum_{q=1}^d \lambda_q \phi_p(x_q) - \beta_0 \right) - \gamma_0 \right) \right\vert < \frac{\varepsilon}{2}
\end{equation}
for all $(x_1, \ldots, x_d) \in [0,1]^d$.

For each $p=1$, $\ldots$, $2d+1$, the function $\phi_p$ in (\ref{eq:Kolmogorov}) is defined on $[a,b]$. For this function, using the linear transformation $x = (t-a)/s$ from $[a,b]$ to $[0,1]$ and the same procedure for the function $g$ above, we can obtain the inequality
\begin{equation} \label{eq:fip}
  \left\vert \phi_p(t) - (\alpha_p \sigma(t - \beta_p) - \gamma_p) \right\vert < \delta,
\end{equation}
for all $t \in [a,b]$. Here $\delta$ is any positive real number, and the parameters $\alpha_p$, $\beta_p$ and $\gamma_p$ depend on $\delta$. Note that these parameters can be computed similarly as in (\ref{eq:coef}).

Since $\lambda_q > 0$ for $q = 1$, $\ldots$, $d$, and $\sum_{q=1}^d \lambda_q = 1$, it follows from (\ref{eq:fip}) that
\begin{equation} \label{eq:sumfip}
  \left\vert \sum_{q=1}^d \lambda_q \phi_p(x_q) - \left( \sum_{q=1}^d \lambda_q \alpha_p \sigma(x_q - \beta_p) - \gamma_p \right) \right\vert < \delta,
\end{equation}
for all $p = 1$, $\ldots$, $2d+1$, and $(x_1, \ldots, x_d) \in [0,1]^{d}$.

Now since the function $\alpha_0 \sigma(s x - \beta_0)$ is uniformly continuous on every closed interval of the real line, we can choose $\delta$ as small as necessary and obtain from (\ref{eq:sumfip}) that
\begin{multline*}
  \left\vert \sum_{p=1}^{2d+1} \alpha_0 \sigma \left( s \sum_{q=1}^d \lambda_q \phi_p(x_q) - \beta_0 \right) \right. \\
  \left. - \sum_{p=1}^{2d+1} \alpha_0 \sigma \left( s \left( \sum_{q=1}^d \lambda_q \alpha_p \sigma(x_q - \beta_p) - \gamma_p \right) - \beta_0 \right) \right\vert < \frac{\varepsilon}{2}.
\end{multline*}
This inequality may be rewritten in the form
\begin{equation} \label{eq:big}
  \left\vert \sum_{p=1}^{2d+1} \alpha_0 \sigma \left( s \sum_{q=1}^d \lambda_q \phi_p(x_q) - \beta_0 \right) - \sum_{p=1}^{2d+1} \alpha_0 \sigma \left( \sum_{q=1}^d c_{pq} \sigma (\mathbf{w}^{q} \cdot \mathbf{x} - \beta_p) - \zeta_p \right) \right\vert < \frac{\varepsilon}{2},
\end{equation}
where $c_{pq} = s \lambda_q \alpha_p$, $\zeta_p = s \gamma_p + \beta_0$, and $\mathbf{w}^{q}$ is the $q$-th coordinate vector. From (\ref{eq:f}) and (\ref{eq:big}) it follows that
\begin{equation} \label{eq:f2}
  \left\vert f(\mathbf{x}) - \left( \sum_{p=1}^{2d+1} \alpha_0 \sigma \left( \sum_{q=1}^d c_{pq} \sigma(\mathbf{w}^{q} \cdot \mathbf{x} - \beta_p) - \zeta_p \right) - (2d+1) \gamma_0 \right) \right\vert < \varepsilon,
\end{equation}
Clearly, the constant $(2d+1) \gamma_0$ can be written in the form
\begin{equation} \label{eq:const}
  (2d+1) \gamma_0 = \alpha \sigma \left( \sum_{q=1}^d c_q \sigma(\mathbf{w}^{q} \cdot \mathbf{x} - \theta_q) - \zeta \right),
\end{equation}
for $c_q = 0$, $q = 1$, $\ldots$, $d$, and suitable coefficients $\alpha$ and $\zeta$. Considering (\ref{eq:const}) in (\ref{eq:f2}) we finally obtain that
\begin{equation*}
  \left\vert f(\mathbf{x}) - \left( \sum_{p=1}^{2d+2} e_p \sigma \left( \sum_{q=1}^d c_{pq} \sigma(\mathbf{w}^{q} \cdot \mathbf{x} - \theta_{pq}) - \zeta_p \right) \right) \right\vert < \varepsilon,
\end{equation*}
where $e_1 = e_2 = \ldots = e_{2d+1}$. The last inequality completes the proof of the theorem.
\end{proof}

\begin{remark}
Obviously, a compact subset $Q$ of the space $\mathbb{R}^{d}$ can be embedded into a box $[-l,l]^{d}$, and by the Tietze extension theorem (see~\cite[Theorem 15.8]{W70}), any continuous function $f$ on $Q$ can be extended to $[-l,l]^{d}$. Hence Theorem~\ref{thm:main} is valid not only for boxes of the form $[a,b]^{d}$ but for any compact set $Q \subset \mathbb{R}^{d}$, with the proviso that $s = 2 l$.
\end{remark}

\begin{remark}
Theorem~\ref{thm:main}, in particular, shows that TLFNs are more powerful than SLFNs, since SLFNs with a fixed number of hidden neurons and/or weights have not the capability of approximating multivariate functions (see Introduction). We refer the reader to \cite{L17} for interesting results and discussions around the comparison of performances between MLFNs and SLFNs.
\end{remark}

\begin{remark}
In \cite{G03}, Gripenberg showed that the general approximation property of feedforward multilayer perceptron networks can be achieved in networks where the number of neurons in each layer is bounded, but the number of layers grows to infinity. This is the case provided the activation function is twice continuously differentiable and not linear. Taking an exceedingly large number of layers is an indispensable part of Gripenberg's method. Can one develop a different method which enables to use only a preliminarily prescribed number of layers for all approximated functions? To answer this question, we started with SLFNs. It turned out that in this case the answer is ``yes'' if approximated functions are univariate (see \cite{GI16}). Moreover, one can fix the weights of constructed SLFNs. But SLFNs with fixed weights or bounded number of neurons are proved not capable of approximating multivariate functions (see \cite{GI18}). Then how many hidden layers with bounded number of neurons are needed to approximate multivariate functions with arbitrary precision? First of all, one may want to know if any such constrained approximation is possible in practice. Theorem~\ref{thm:main} shows that even two hidden layers and a specifically constructed activation function are sufficient to solve this problem affirmatively.
\end{remark}

\end{document}